\documentclass[twoside]{article}
\usepackage[utf8]{inputenc}

\usepackage{fullpage}
\usepackage{multirow,tabularx}

\usepackage{fancyhdr}
\usepackage[T1]{fontenc}
\usepackage{parskip}

\usepackage{algorithm}
\usepackage{algorithmic}
\usepackage{comment}
\usepackage{hyperref}
\usepackage{xcolor}
\definecolor{Ckt}{HTML}{E41A1C}
\definecolor{Min}{HTML}{4D4D4D}
\definecolor{MinMore}{HTML}{377EB8}
\definecolor{Data}{HTML}{984EA3}
\definecolor{Rpurple}{HTML}{A020F0}

\usepackage{tikz}
\usetikzlibrary{arrows}
\usetikzlibrary{fit,positioning}
\usepackage{amssymb,amsmath}
\usepackage{natbib}
\usepackage{amsthm}

\newtheorem{theorem}{Theorem}

\DeclareMathOperator*{\argmin}{arg\,min}

\usepackage{stfloats}
\DeclareMathOperator*{\Diag}{Diag}

\newcommand{\ZZ}{\mathbb Z}

\newcommand{\RR}{\mathbb R}

\begin{document}

\title{Linear time dynamic programming for the exact path of optimal models selected from a finite set}

\author{Toby Hocking, Joseph Vargovich}
\maketitle
\begin{abstract}
Many learning algorithms are formulated in terms of finding model parameters which minimize a data-fitting loss function plus a regularizer. 
When the regularizer involves the $\ell_0$ pseudo-norm, the resulting regularization path consists of a finite set of models.
The fastest existing algorithm for computing the breakpoints in the regularization path is quadratic in the number of models, so it scales poorly to high dimensional problems. We provide new formal proofs that a dynamic programming algorithm can be used to compute the breakpoints in linear time. Empirical results on changepoint detection problems demonstrate the improved accuracy and speed relative to grid search and the previous quadratic time algorithm. 
\end{abstract}



\section{Introduction}
\label{sec:intro}
In this paper we propose a new algorithm related to the regularization path of machine learning problems such as

\begin{equation}
  \label{eq:theta_hat}
  \hat \theta(\lambda) =  \argmin_\theta \mathcal L(\theta) + \lambda R(\theta),
\end{equation}
where $\theta\in\mathbb R^p$ is a vector of model parameters, the loss function $\mathcal L:\mathbb R^p \rightarrow \mathbb R$ is typically convex, and $\lambda\geq 0$ is a penalty constant. 
The regularizer $R:\mathbb R^p\rightarrow \mathbb R_+$ is a non-convex function involving the $\ell_0$ pseudo-norm, 
\begin{equation}
    ||\theta||_0 = \sum_{j=1}^p I[\theta_j \neq 0]\in \mathbb Z_+ = \{0,1,\dots,p\}
\end{equation}
which counts the number of non-zero entries of the $\theta$ parameter vector ($I$ is the indicator function).
Some typical examples are given in Table~\ref{tab:examples}, which includes best subset regression \citep{Miller2002,Soussen2010}, optimal segmentation \citep{segment-neighborhood,optimal-partitioning}, $k$-means clustering \citep{macqueen1967}, and low-rank matrix factorization \citep{Huang2018}.
For learning it is important to compute not just a single model for one penalty $\lambda$, but also the full regularization path $\{\hat \theta(\lambda)|\lambda \geq 0\}=\{\Theta_1, \Theta_2, \dots, \Theta_{N}\}$. The path is a finite set of $N$ models, i.e. for any $\lambda\geq 0$, we have $\hat \theta(\lambda)=\Theta_k$ for some model size $k\in\{1,\dots, N\}$. To simplify the presentation we limit our discussion to regularizers $R(\Theta_k)=k$ which are equal to model size. More general regularizers, e.g. $R(\Theta_k)=r_k$ for some sequence of increasing values $r_1<\cdots<r_N$, can be handled using a straightforward modification of our proposed algorithm.

One example is best subset regression, which seeks the best $k$ features for a linear regression function. 
In this problem there a total of $p$ input features and therefore a set of $N=p+1$ models in the regularization path $\{\Theta_0, \Theta_1,\dots, \Theta_p\}$. 
This non-convex problem is NP-hard, so the optimal regularization path can only be computed for low-dimensional problems \citep{Bertsimas2016}. 
For high-dimensional problems there are various heuristic algorithms for computing approximate solutions, e.g. greedy forward/backward selection \citep{MallatZhang1993,Davis1994,Miller2002,Schniter2009,Soussen2010}, non-smooth non-convex regularizers \citep{FanLi2001,Zhang2010,Mazumder2011,SparseStep}, and $\ell_1$ regularization/LASSO \citep{Tibshirani1996,Chen1998}. 
Each weight vector $\Theta_k\in\mathbb R^p$ in the optimal or approximate regularization path has $k$ non-zero entries, $R(\Theta_k)=||\Theta_k||_0=k$. 
The extreme elements correspond to the ordinary least squares solution $\Theta_p$ with all features selected, and the completely regularized solution $\Theta_0$ with no features selected. 

Another example is optimal segmentation, which is the maximum likelihood model with $k$ segments ($k-1$ changepoints) for a sequential data set. 
In this problem there are $p$ sequence data, and each element $\Theta_k$ of the regularization path has $k\in\{1,\dots,p\}$ distinct segments ($k-1$ changes) along the sequence, i.e. $R(\Theta_k)=1+||D\Theta_k||_0=k$ where $D\in\mathbb R^{p-1\times p}$ is the matrix which returns the difference between adjacent pairs of data in the sequence.
Even though this problem is non-convex, an optimal solution can be computed via dynamic programming algorithms that are log-linear in the number of data, and linear in the number of models \citep{pelt,Maidstone2016}. 
There are also several fast heuristic algorithms, including binary segmentation \citep{binary-segmentation,Truong2018}, which computes an approximate regularization path of $N=p$ models in $O(N\log N)$ time on average and $O(N^2)$ in the worst case.
The optimal or approximate regularization path of $N=p$ models $\{\Theta_1, \Theta_2, \dots, \Theta_{p}\}$ has extreme elements $\Theta_1$ with no changes (one common segment/parameter for the entire data sequence) and $\Theta_{p}$ with a change after every data point (a different segment/parameter for each data point). 


\begin{table*}[t]
    \centering
    \begin{tabular}{cccc}
        Problem &  Loss  & Regularizer & Model complexity  \\
            \hline
        Best Subset Regression & $||X\theta-y||_2^2$ & $||\theta||_0$ & features selected \\
        Optimal Segmentation & $||\theta-y||^2_2$ & $|| D\theta ||_0$ & segments/changepoints\\
        $k$-means Clustering & $||\theta M - X||_F$ & $||\theta^T \mathbf 1||_0$ & cluster centers\\
        Matrix Factorization & $||U\Diag(\theta) V^T-X||_F$ & $||\theta||_0$ & rank\\[0.1cm]
    \end{tabular}
    \caption{Examples of machine learning problems which use the $\ell_0$ pseudo-norm as a regularizer.}
    \label{tab:examples}
\end{table*}

In this context solving the penalized problem~(\ref{eq:theta_hat}) for a given penalty $\lambda\geq 0$ results in one of the solutions to the corresponding constrained problem,
\begin{equation}
  \label{eq:L_k_theta}
   L_k=\min_{\theta}
      \mathcal L(\theta), 
  \text{ subject to } R( \theta) \leq k,
\end{equation}
where $k\in\mathbb Z_+$ is the model size (selected features, changepoints, clusters, etc). 
If the loss values $L_1>\dots>L_N$ in a regularization path $\{\Theta_1, \Theta_2, \dots, \Theta_{N}\}$ are known, they can be used to define the model selection function
\begin{equation}
\label{eq:k*}
k_N^*(\lambda) = \argmin_{k\in\{1,\dots, N\} }
\underbrace{L_k + \lambda k}_{f_k(\lambda)}.
\end{equation}
The model selection function~(\ref{eq:k*}) returns the (smallest) model complexity $k$ which is optimal for a given penalty parameter $\lambda$. 
In this paper we provide a new formal proof that dynamic programming can be used to compute an exact representation of the model selection function $k^*_N(\lambda)$ in linear $O(N)$ time.

\subsection{Existing algorithms and related work}
The model selection function $k^*_N(\lambda)$ can be trivially evaluated for a single $\lambda$ parameter in linear $O(N)$ time, which yields the solution to (\ref{eq:theta_hat}) via $\hat \theta(\lambda) = \Theta_{k^*_N(\lambda)}$.
However for some learning algorithms we need an exact representation of the model selection function for a full path of penalty $\lambda$ values.
Quadratic $O(N^2)$ time algorithms for computing the full path have been proposed for changepoint detection \citep{Lavielle2005,HOCKING-penalties} and regression \citep{Arlot2009}, but these algorithms are too slow for high-dimensional problems (e.g. full path of binary segmentation models for large data sets, see Section~\ref{sec:empirical-timings}).

The algorithm we propose is similar to the ``convex hull trick'' which is informally described, without any references to the machine learning literature, on a web page 
\citep{Convex-hull-trick}. 
The novelty of our paper with respect to that previous work is (1)  rigorous formal proofs of the linear time complexity and optimality, (2) explaining the relevance to the machine learning literature, (3) detailed theoretical and empirical comparisons with baseline algorithms.

A final related work is the CROPS algorithm of \citet{crops}, which also
proposes an algorithm that outputs an exact path of solutions for several penalty values. Both papers exploit the structure of the piecewise linear model selection function which relates the constrained and penalized problems. The input to our algorithm is a sequence of constrained models of sizes 1 to $N$, whereas the input to CROPS is an interval of penalty values. In general the two algorithms output different results (partial solution paths).
However in the special case when $N=p$ models are input to our algorithm (all possible models) and the interval $(0,\infty)$ is input to CROPS, then the two algorithms return the same output (the full path).


\subsection{Important definitions}
In this section, we will reference several concepts that are related to our proposed algorithm. Changepoints refer to a sudden deviation from previously recorded values within a given data set. Segments are used to fit each unique block data with an average value line. The model complexity $(k)$ refers to the total amount of unique segments present in a given segmentation model. A regularization path is the path of optimal segmentation models selected for a given range of penalties. Breakpoints represent the penalty value in which the algorithm switches from a previously selected model to a new one based on the evaluation of the model selection function (\ref{eq:k*}).

\subsection{Contributions and organization}

In this paper we propose a new dynamic programming algorithm for computing the model selection function $k^*_N(\lambda)$, and we prove that it computes an exact representation for all penalties $\lambda\geq 0$ (Section~\ref{sec:algorithm}).
Our second contribution is a theoretical analysis of the time complexity of our algorithm, which demonstrates that it is linear $O(N)$ time in the worst case; we also provide a theoretical analysis of previous algorithms in terms of the framework of this paper  (Section~\ref{sec:pseudocode}). Our third contribution is an empirical study of time complexity in several real and synthetic data sets, including a comparison with previous algorithms (Section~\ref{sec:empirical-complexity}). Our final contribution is an empirical study of the prediction accuracy in cross-validation experiments on supervised changepoint detection problems (Section~\ref{sec:accuracy}). The paper concludes with a discussion (Section~\ref{sec:discussion}).

\section{Dynamic programming algorithm}
\label{sec:algorithm}
We propose a dynamic programming algorithm for $N$ decreasing loss values $L_1>\cdots>L_N$; it computes an exact representation of the $k^*_N(\lambda)$ model selection function for all penalties $\lambda\geq 0$. 

\subsection{Exact representation of a piecewise constant function using breakpoints}
\label{sec:breakpoints}
Our proposed algorithm recursively computes $k^*_N$ from $k^*_{N-1}$, so is an instance of dynamic programming \citep{bellman}. At each step/iteration $t\in\{1,\dots, N\}$ of the algorithm, the algorithm stores a set of $M_t\in\{1,\dots,t\}$ selectable models,
\begin{equation}
\label{eq:K}
     1=K_{t,1}<K_{t,2}<\cdots<K_{t,M_t}=t.
\end{equation}
The algorithm also stores a corresponding set of breakpoints,
\begin{equation}
    \infty=b_{t,0} > b_{t,1} > \cdots > b_{t,M_t}=0.
\end{equation}
These two sets define for all $t\geq 1$ a recursively computed model selection function,
\begin{equation}
    F_t(\lambda) = \begin{cases}
     K_{t,1} & \text{ if } \lambda\in ( b_{t,1}, b_{t,0} )\\
     &\vdots\\
     K_{t,M_t} & \text{ if } \lambda\in ( b_{t,M_t}, b_{t,M_t-1} )
    \end{cases}
    \label{eq:F_t}
\end{equation}
We prove later (Theorem~\ref{thm:optimal}) that the recursively computed function $F_N$ is identical to $k^*_N$, the desired model selection function~(\ref{eq:k*}). The geometric interpretation of the models $K_{t,i}\in\mathbb Z$ and breakpoints $b_{t,i}\in \mathbb R$ are shown in
Figure~\ref{fig:three-iterations}. Each breakpoint is a penalty value where the min cost (grey segments) changes from one cost function to another (black lines). 
\begin{figure*}[!b]
  \centering
  \input{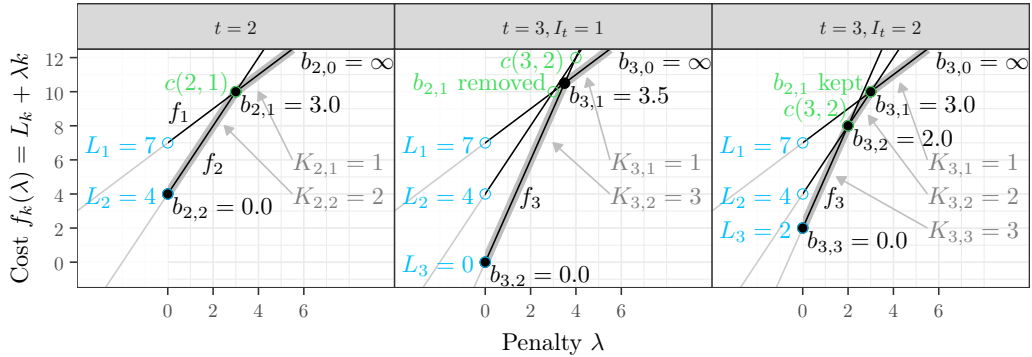}
  \vskip -.5cm
  \caption{Two possible runs of the algorithm for $N=3$ models. \textbf{Left:} at iteration $t=2$, the cost function $f_2(\lambda)$ is minimal (grey curve) for penalties $\lambda<b_{2,1}=c(2,1)$, and $f_1(\lambda)$ is minimal otherwise. \textbf{Middle:} at iteration $t=3$ a small $L_3$ value results in $c(3,2)>b_{2,1}$ so we remove $b_{2,1}$ and $K_{2,2}=2$, because $f_2(\lambda)$ is no longer minimal for any $\lambda$. \textbf{Right:} at iteration $t=3$ a large $L_3$ value results in $c(3,2)<b_{2,1}$ so we keep $b_{2,1}=b_{3,1}$ and store a new breakpoint $c(3,2)=b_{3,2}$.}
  \label{fig:three-iterations}
\end{figure*}
\subsection{Dynamic programming update rules}
\label{sec:updates}
The algorithm starts at the first iteration $t=1$ by initializing $M_1=1$ model with 
\begin{equation}
\label{eq:init}
K_{1,1}=1,\, b_{1,1}=0,\, b_{1,0}=\infty,
\end{equation}
which is an exact representation of the first model selection function $F_1$. For all other iterations $t>1$, the algorithm begins by discarding any breakpoints which are no longer necessary, then adds one new breakpoint. In particular it first computes the index $i$ corresponding to the largest model from the previous iteration $t-1$ which is still selected at iteration $t$,
\begin{equation}
    I_t = \max\big\{i\in\{1,\dots,M_{t-1}\} \mid 
    c(t, i) < b_{t-1,i-1} 
    \big\}.
\label{eq:It}
\end{equation}
The new candidate breakpoint $c(t,i)$ is 
the penalty for which the new cost function $f_t$ is equal to a previous function $f_{K_{t-1,i}}$,
\begin{equation}
\label{eq:cti}
     c(t,i) = \frac{L_{K_{t-1,i}} - L_t}{t - K_{t-1,i}}.
\end{equation}

A modification to use more general regularizers involves replacing $t - K_{t-1,i}$ with $r_{t}-r_{K_{t-1,i}}$ in the denominator of equation~(\ref{eq:cti}).

We then compute the number of models selected 
at iteration $t$,
\begin{equation}
    M_t = I_t + 1.
\end{equation}
The algorithm stores the new set of models selected at iteration $t$,
\begin{equation}
\label{eq:Kti}
    K_{t,i} = \begin{cases}
    K_{t-1,i} & \text{ for } i\in\{1,\dots, I_t\}\\
    t & \text{ for } i=M_t.
    \end{cases}
\end{equation}
The algorithm also stores the new breakpoint $c(t,I_t)$ along with some of the previous breakpoints,
\begin{equation}
    b_{t,i} = \begin{cases}
    b_{t-1,i} & \text{ for } i\in\{0,\dots, I_{t}-1\}\\
    c(t, I_t) & \text{ for } i=I_t\\
    0 & \text{ for } i=M_t.
    \end{cases}
    \label{eq:bti}
\end{equation}
Once the selected models $K_{t,i}$ and breakpoints $b_{t,i}$ have been recursively computed via (\ref{eq:Kti}--\ref{eq:bti}), the model selection function $F_t$ is defined using equation~(\ref{eq:F_t}).

\subsection{Demonstration of algorithm up to $t=3$}
In this section we provide two example runs of the algorithm.
The initialization creates an exact representation of $k^*_1=F_1$, via one possible model $K_{1,1}=1$ which is selected for all $\lambda$ between the two breakpoints $b_{1,1}=0<\infty=b_{1,0}$. The second iteration computes the candidate breakpoint $c(2,1)$ which is stored in the new breakpoints $b_{2,2}=0<c(2,1)<\infty=b_{2,0}$, along with $M_2=2$ models $K_{2,2}=2>1=K_{2,1}$ (Figure~\ref{fig:three-iterations}, left).

At iteration $t=3$ we first compute the candidate $c(3,2)$ and compare it to the stored breakpoint $b_{2,1}$. If $c(3,2)\geq b_{2,1}$ then $I_3=1$ and so $b_{2,1}$ is removed (Figure~\ref{fig:three-iterations}, middle). 
The candidate $c(3,2)$ is also discarded; the new breakpoints $b_{3,2}=0<c(3,1)<\infty=b_{3,0}$ are computed and stored with $M_3=2$ models $K_{3,2}=3>1=K_{3,1}$. Otherwise, $c(3,2)< b_{2,1}$ implies $I_3=2$ and the previous breakpoint $b_{2,1}$ is kept along with the candidate $c(3,2)$  (Figure~\ref{fig:three-iterations}, right). The breakpoints $b_{3,3}=0<c(3,2)<b_{2,1}<\infty=b_{3,0}$ are stored with $M_3=3$ models $K_{3,3}=3>2>1=K_{3,1}$.

\subsection{Recursive update rules are optimal}
Equations~(\ref{eq:K}--\ref{eq:bti}) define a dynamic programming algorithm, because the recursively computed $F_N$ is optimal in the sense of equation~(\ref{eq:k*}), as demonstrated in the following theorem.
\begin{theorem}[Update rules yield the optimal model selection function]
\label{thm:optimal}
The recursively computed function $F_N$~(\ref{eq:F_t}) and the model selection function $k^*_N$~(\ref{eq:k*}) are identical.
\end{theorem}
\begin{proof}
The proof follows from equations~(\ref{eq:K}--\ref{eq:bti}) using induction on $t$. The base case is $t=1$, for which the initialization~(\ref{eq:init}) of the recursively computed function implies $F_1(\lambda) = 1$ for all $\lambda\in (0, \infty)$. Because at iteration $t=1$ there is only one possible model, it is clear that $F_1(\lambda)=k^*_1(\lambda)$ for all $\lambda$.

The proof by induction now assumes that $F_{t-1}(\lambda)=k^*_{t-1}(\lambda)$ for all $\lambda$; we will prove that the same is true for $t$. 
The recursive updates (\ref{eq:Kti}--\ref{eq:bti}) imply that
\begin{eqnarray}
F_{t}(\lambda) 
&=& \begin{cases}
    K_{t-1,1} & \text{ if } \lambda\in( b_{t-1,1}, b_{t-1,0})\\
    & \vdots\\
    K_{t-1,I_t} & \text{ if } \lambda \in ( c(t,I_t), b_{t-1, I_t-1} )\\
    t & \text{ if } \lambda \in (0, c(t,I_t) )
    \end{cases}\\
&=& \begin{cases}
    F_{t-1}(\lambda) & \text{ if } \lambda > c(t,I_t)\\
    t & \text{ if } \lambda < c(t,I_t)
    \end{cases}
    \label{eq:Ft_minus_one}
\end{eqnarray}
We need to prove that the function above returns the model size $k\in\{1,\dots, t\}$ with min cost $f_k(\lambda)$, for any penalty $\lambda$. Equations~(\ref{eq:It}--\ref{eq:cti}) imply
$K_{t-1,I_t}=F_{t-1}[c(t, I_t)]$ is the min cost model at the penalty $c(t, I_t)$ where the new cost function $f_t$ equals the previous min cost function,
\begin{equation}
    f_t[c(t,I_t)]=f_{K_{t-1,I_t}}[c(t,I_t)] = \min_{k\in\{1,\dots,t-1\}} f_k[c(t,I_t)].
\end{equation}
Because $f_t(\lambda)=L_t+\lambda t$ is a linear function with a larger slope than any of $f_1,\dots,f_{t-1}$, and a smaller intercept $L_t<L_{t-1}<\cdots$, we therefore deduce that $f_t$ is less costly before $c(t,I_t)$, and more costly after:
\begin{equation}
    \begin{cases}
    f_t(\lambda)< \min_{k\in\{1,\dots,t-1\}} f_k(\lambda) & \text{ for all } \lambda < c(t,I_t) \\ 
    f_t(\lambda)> \min_{k\in\{1,\dots,t-1\}} f_k(\lambda) & \text{ for all } \lambda > c(t,I_t) 
    \end{cases}
    \label{eq:ft_greater_less}
\end{equation}
Combining equations~(\ref{eq:Ft_minus_one},\ref{eq:ft_greater_less}) and using the induction hypothesis completes the proof that $F_t(\lambda)=k^*_t(\lambda)=\argmin_{k\in\{1,\dots, t\}} f_t(\lambda)$ for all $\lambda$.
\end{proof}

\section{Theoretical complexity analysis}
\label{sec:pseudocode}


In this section we propose pseudocode that efficiently implements the dynamic programming algorithm, and provide a proof of worst case linear time complexity. We also provide a theoretical analysis of the previous quadratic time algorithm in terms of the framework of this paper.

\subsection{Proposed linear time algorithm}

We propose Algorithm~\ref{algo:dp}, which is pseudocode for equations (\ref{eq:It}--\ref{eq:bti}). 
It recursively computes an exact representation of the model selection function $F_N$ in terms of breakpoints $b$ and selected models $K$. 

It begins by initializing the model selection function $F_1$ (line~\ref{line:init}). 
Then for all $t\in\{2,\dots,N\}$ it recusively computes $F_t$ from $F_{t-1}$. 
The first step in the loop (line~\ref{line:solve}) is to call the \textsc{Solve} sub-routine, which computes the number of selected models $M_t$ and the new breakpoint $\lambda=c(t,M_t-1)$. The number of while loop evaluations $w[t]$ can optionally be stored in order to analyze empirical time complexity. 
The next step is to store the new model $t$ and new breakpoint $\lambda$ (line~\ref{line:store}), which completes the computation of $F_t$.

\begin{table*}[t]
    \centering
        \begin{tabular}{ccc}
        Algorithm &  Best & Worst  \\
            \hline
        \textbf{This paper}, Algorithm~\ref{algo:dp} & $O(N)$ & $O(N)$\\
        \citep{Arlot2009, HOCKING-penalties}& $O(N)$ & $O(N^2)$ \\
        Always quadratic & $O(N^2)$ & $O(N^2)$
    \end{tabular}
    \caption{Summary of asymptotic time complexity in terms of number of input models, $N$.}
    \label{tab:complexity}
\end{table*}

\begin{algorithm}
\begin{algorithmic}[1]
\STATE Input: 
Array of $N$ real numbers $L[1]>\cdots>L[N]$ (decreasing loss values).
\STATE 
\label{line:allocate} Allocate: selected models $K\in\ZZ^N$, breakpoints $b\in\RR^N$, while loop iterations $w\in\ZZ^N$
\STATE \label{line:init} Initialize: number of models $M\gets 1$, breakpoint $b[1]\gets \infty$, selected model $K[1]\gets 1$
\FOR{$t=2$ to $N$}
  \STATE \label{line:solve} $M,\lambda,w[t]\gets\textsc{Solve}(K, b, L, M, t)$
  \STATE \label{line:store} $b[M]\gets \lambda$, $K[M]\gets t$\label{line:K}  // store a new breakpoint
\ENDFOR
\STATE Output: selected models $K[1:M]$, breakpoints $b[1:M]$, while loop iterations $w[2:N]$
\caption{\label{algo:dp}
Dynamic programming for computing exact representation of model selection function}
\end{algorithmic}
\end{algorithm}

In this paper we propose an amortized constant $O(1)$ time implementation of the \textsc{Solve} sub-routine (Algorithm~\ref{algo:proposed-solve}). 
It computes $I_t,M_t$ by solving the maximization in equation~(\ref{eq:It}) using a linear search over possible values of the model index $i$. 
It starts at the current number of selected models (line~\ref{line:initM}), and then repeatedly tests the criterion from equation~(\ref{eq:It}). 
If the current value of the model index $i$ does not satisfy the condition of the while loop (line~\ref{line:while}), then the model index is decremented to remove a breakpoint (line~\ref{line:remove}). 
The number of while loop iterations $w_t$ (lines~\ref{line:initM},\ref{line:remove}) can be optionally computed in order to analyze the empirical time complexity of the algorithm. 
Even though Algorithm~\ref{algo:proposed-solve} is clearly $O(M)$ in the worst case, in the next section we prove that it is amortized constant $O(1)$ time when used in the context of Algorithm~\ref{algo:dp}. Using this sub-routine therefore results in an overall linear $O(N)$ time complexity for Algorithm~\ref{algo:dp}, in the best and worst case  (Table~\ref{tab:complexity}).

\begin{algorithm}
\begin{algorithmic}[1]
\STATE Input: selected model sizes $K\in\mathbb Z^N$, breakpoints $b\in\mathbb R^N$, loss values $L\in\mathbb R^N$, number of selected models $M\in\mathbb Z$, new model size $t\in\mathbb Z$.
\STATE \label{line:initM} $i\gets M,\  w_t\gets 1$
\WHILE{$\lambda\gets (L[ K[i] ] - L[t])/(t- K[i]) \geq b[i]$}
\label{line:while}
    \STATE \label{line:remove} $i--,\  w_t ++$ // remove a breakpoint
\ENDWHILE
\STATE Output: number of models $i+1$, new breakpoint $\lambda$, number of while loop iterations $w_t$
\caption{\label{algo:proposed-solve}
Proposed \textsc{Solve} sub-routine
}
\end{algorithmic}
\end{algorithm}

\subsection{Previous quadratic algorithms}
\label{sec:quadratic-pseudocode}

In this section we provide a detailed comparison with several previously proposed quadratic algorithms \citep{Arlot2009,HOCKING-penalties}. 
In terms of the framework of this paper, these previous algorithms can be interpreted as computing $K_{t,i}, b_{t,i}$ for $t=N$, without computing any of the solutions at the previous iterations $t<N$. 
These other algorithms are therefore not performing dynamic programming. 
Whereas our algorithm starts at the smallest model size and then updates the model selection function for larger sizes, these other algorithms begin at the largest model size. 
In particular they start by initializing the largest model $K_{N,M_N}=N$ and the smallest breakpoint $b_{N,M_N}=0$, then for all $i\in\{M_N-1, \dots, 1\}$ they recursively compute $K_{N,i},b_{N,i}$ from $K_{N,i+1},b_{N,i+1}$. 
There are $M_N-1$ iterations of this recursive computation, and each iteration considers $K_{N,i+1}-1$ breakpoints. The overall algorithm is therefore $O(N M_N)$; best case $O(N)$ is when the number of selected models $M_N=2$ is small; worst case $O(N^2)$ is when $M_N=N$ is large (Table~\ref{tab:complexity}). Interestingly, the opposite is true of our algorithm (best case is when $M_N$ is large), as we prove in the next section.

\subsection{Proof of linear time and space complexity}
\label{sec:time-proof}
The overall space complexity of Algorithm~\ref{algo:dp} is clearly $O(N)$, because up to $N$ possible models/breakpoints can be computed. 
The time complexity depends on the implementation of the \textsc{Solve} sub-routine (line~\ref{line:solve}). 

The computation time of our proposed implementation of the \textsc{Solve} sub-routine (Algorithm~\ref{algo:proposed-solve}) depends on $w_t$, the number of times the while condition is evaluated (line~\ref{line:while}). 
In particular, the overall time complexity of Algorithm~\ref{algo:dp} is linear in total number of times the while condition is checked, 
\begin{equation}
\label{eq:WN}
    W_N = \sum_{t=2}^N w_t.
\end{equation}
The following result proves that Algorithm~\ref{algo:dp} is overall $O(N)$ time, by bounding the total number of times the while condition is checked.
\begin{theorem}[Best and worst case time complexity]
\label{thm:time-complexity}
For any $N$ inputs to Algorithm~\ref{algo:dp}, the total number of while loop iterations $W_N$ over all calls to Algorithm~\ref{algo:proposed-solve} is bounded: $N-1\leq W_N \leq 2N-3$.
\end{theorem}
\begin{proof}
The proof uses the fact that for all $t\in\{2,\dots, N\}$, we have
\begin{equation}
    \label{eq:model_size_iterations} M_t = 2+M_{t-1}-w_t,
\end{equation}
which follows from the definition of the number of while loop iterations $w_t$ (on line~\ref{line:remove} of Algorithm~\ref{algo:proposed-solve}, every iteration decrements $i$, and therefore $M_t$). The total number of while loop iterations is thus
\begin{eqnarray}
W_N 
\label{eq:sum_w} &=& \sum_{t=2}^N w_t = \sum_{t=2}^N 2+M_{t-1} - M_t\\
\label{eq:distribute_sum} &=& 2(N-1)+ \sum_{t=2}^{N} M_{t-1}-\sum_{t=2}^N M_t \\
\label{eq:re_write_sum} &=& 2(N-1)+ \sum_{t=1}^{N-1} M_{t}-\sum_{t=2}^N M_t\\
\label{eq:subtract_two_sums} &=& 2N-2  + M_1 - M_N \\
\label{eq:M1_is_one} &=& 2N-1 - M_N.
\end{eqnarray}
The first two equalities (\ref{eq:sum_w}) follow from the definitions of the number of while loop iterations (\ref{eq:WN}--\ref{eq:model_size_iterations}). 
The next equalities come from distributing the sum (\ref{eq:distribute_sum}), then re-writing the second term as a sum from $t=1$ to $N-1$ (\ref{eq:re_write_sum}). 
The last equalities come from subtracting the terms in the two sums (\ref{eq:subtract_two_sums}), then using the fact that $M_1=1$ (\ref{eq:M1_is_one}). 
The result is obtained using the fact that the number of selectable models is bounded, $2\leq M_N\leq N$.
\end{proof}

The best case of Algorithm~\ref{algo:dp}, $W_N=N-1$ iterations, happens when the number of selected models is large, $M_N=N$; the worst case $W_N=2N-3$ iterations occurs when $M_N=2$.
Because the total number of iterations is always $O(N)$, the \textsc{Solve} sub-routine (Algorithm~\ref{algo:proposed-solve}) is amortized constant $O(1)$ time on average, even though it is linear in the number of models $O(M)$ in the worst case. 

\begin{figure}
  \centering
  \input{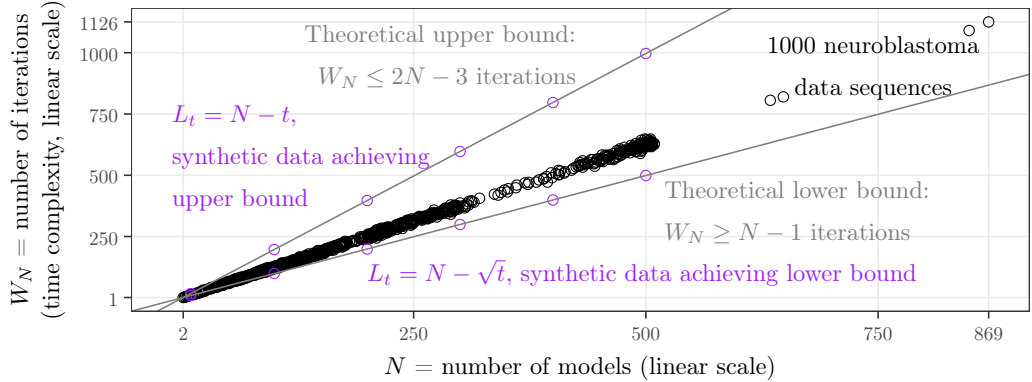}
  \vskip -0.5cm
  \caption{Empirical number of iterations $W_N$ of while loop in Algorithm~\ref{algo:proposed-solve} in synthetic loss values (\textcolor{Rpurple}{violet points}) and loss values from optimal changepoint models of real neuroblastoma data (black points) are linear $O(N)$ in the number of input models $N$, consistent with theoretical upper/lower bounds obtained in Theorem~\ref{thm:time-complexity} (\textcolor{gray}{grey lines}). }
  \label{fig:empirical-iterations}
\end{figure}

\section{Empirical complexity analysis}
\label{sec:empirical-complexity}

\label{sec:empirical-timings}
\begin{figure*}
  \centering
   \includegraphics[height=2.8in]{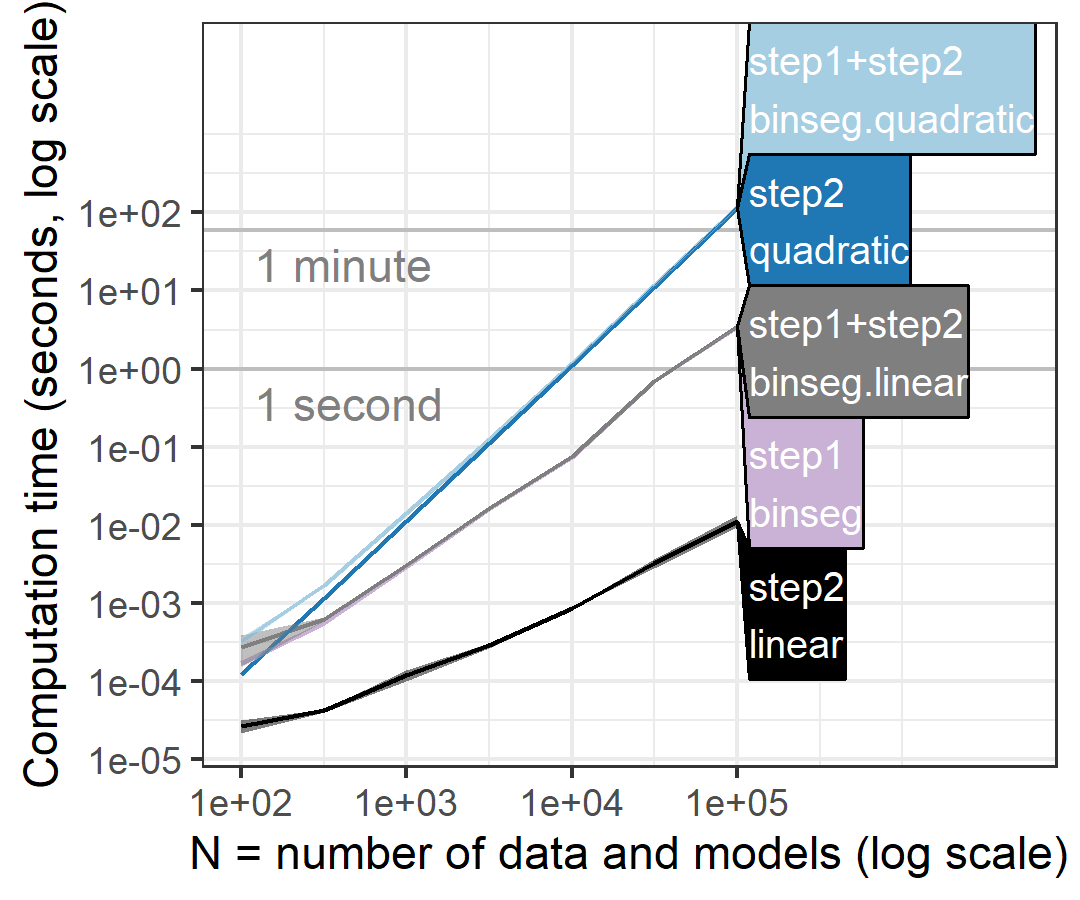}
   \includegraphics[height=2.8in]{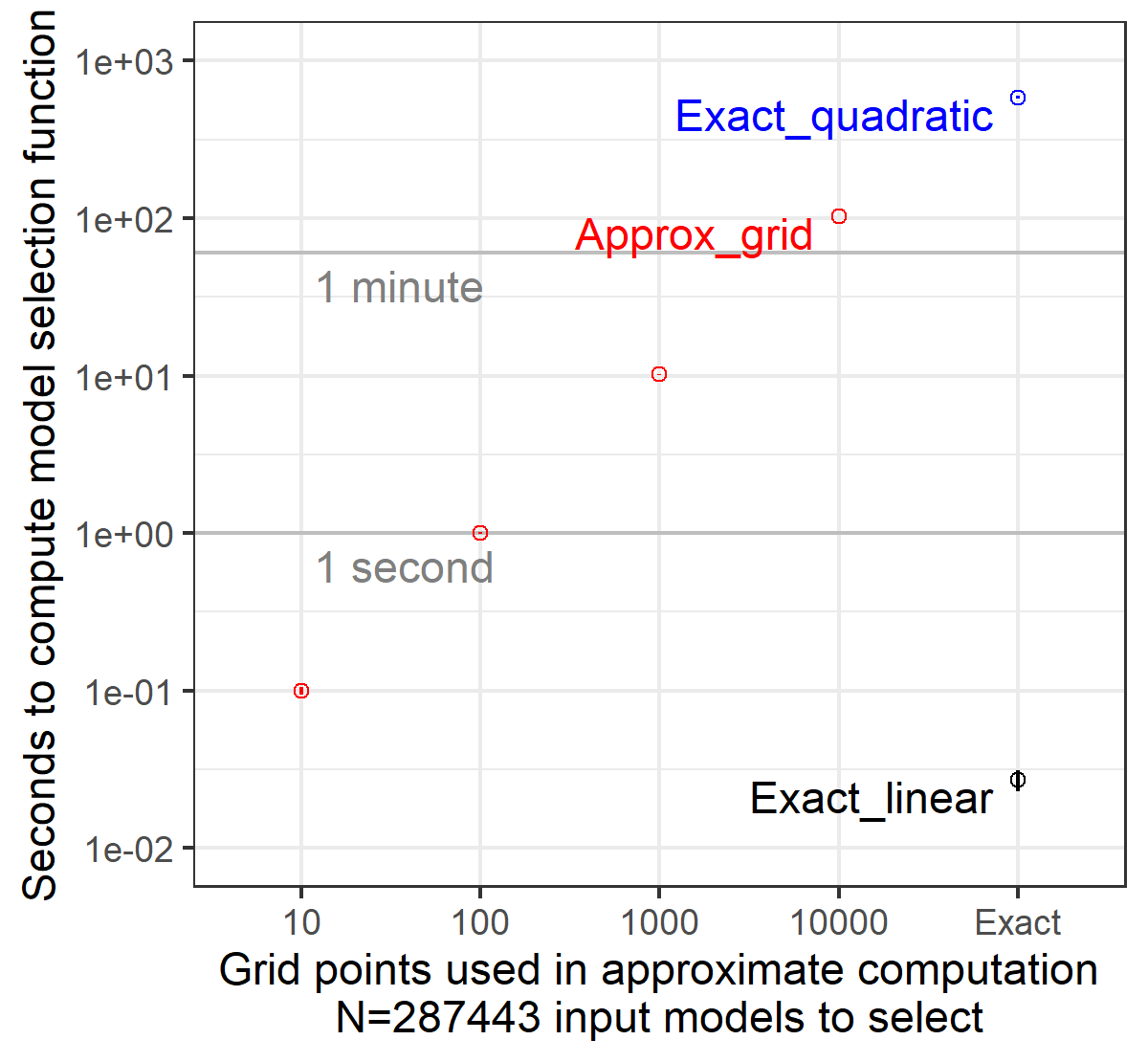}
  \vskip -0.5cm
  \caption{The proposed exact linear time algorithm is orders of magnitude faster than the previous exact quadratic time algorithm
  and naïve approximate grid search. 
  \textbf{Left:} the two exact algorithms compute the same result, but the proposed linear time algorithm is orders of magnitude faster, even when time to compute loss values via binary segmentation (binseg) is included in the timing (lines/bands for mean/SD over 5 timings).  
  \textbf{Right:} when used on loss values from $N=287443$ optimal changepoint models for one genomic data sequence, the proposed exact linear time algorithm is always faster than approximate grid search with at least 10 points (points/segments for mean/SD over 5 timings). }
  \label{fig:empirical-time}
\end{figure*}

In this section we empirically examine the number of iterations of our algorithm, and show that it is overall orders of magnitude faster than previous baselines.

\subsection{Empirical iteration counts are consistent with theoretical bounds}

As discussed in Section~\ref{sec:time-proof}, the time complexity of Algorithm~\ref{algo:dp} is linear $W_N$, the total number of iterations of the while loop in the \textsc{Solve} sub-routine. Here we demonstrate that the theoretical bounds on $W_N$ obtained in Theorem~\ref{thm:time-complexity} are consistent with the number of iterations obtained empirically in real and synthetic data.
First, we considered 1000 real cancer DNA copy number data sets of different sizes $p\in\{2, \dots,869\}$ from R package \texttt{neuroblastoma}. 
For each sequence data set $ \mathbf z\in \mathbb R^p$ we used the Pruned Dynamic Programming Algorithm (PDPA) of \citet{pruned-dp-new} to compute a sequence of optimal changepoint models. For each number of segments $k\in\{1,\dots, p\}$ the optimal loss is
\begin{align}
  \label{eq:L_k_cpt}
   L_k=&\ \ \min_{\mathbf \theta\in\RR^p}\ 
      \sum_{j=1}^p (\theta_j-z_j)^2\\
      \text{subject to} &\ \ 
      ||D\theta||_0=\sum_{j=1}^{p-1} I[\theta_j\neq \theta_{j+1}]=k-1.
      \nonumber
\end{align}
The PDPA returns a regularization path of $N=p$ models, from $k=1$ segment (no changepoints, $\theta_j=\theta_{j+1}$ for all $j$) to $k=N=p$ segments (change after every data point, $\theta_j\neq\theta_{j+1}$ for all $j$). We used the resulting loss values $L_1>\cdots>L_N$ as input to Algorithm~\ref{algo:dp}. We plotted the number of iterations $W_N$ as a function of data set size $N$ (black points in Figure~\ref{fig:empirical-iterations}), and observed that they always fall between the upper/lower bounds from Theorem~\ref{thm:time-complexity} (grey lines). These results provide empirical evidence that the time complexity of our algorithm is linear $O(N)$ in real data.

Second, we considered two synthetic sequences of loss values, $L_t=N-t$ for all $t\in\{1,\dots,N\}$ (e.g. $L_1=4>3>2>1>0=L_5$ for $N=5$) and $L_t=N-\sqrt{t}$ (e.g. $L_1=4>3.59>3.27>3>2.76=L_5$ for $N=5$). For these loss values we observed a number of iterations 
(violet points in Figure~\ref{fig:empirical-iterations}) that always falls on the upper/lower bounds (grey lines), which indicates that these synthetic data achieve the worst/best case. Overall these results provide a convincing empirical validation of our theoretical bounds from Theorem~\ref{thm:time-complexity}.

\subsection{Empirical timings suggest orders of magnitude speedups}

\begin{figure*}[!b]
    \centering
    \input{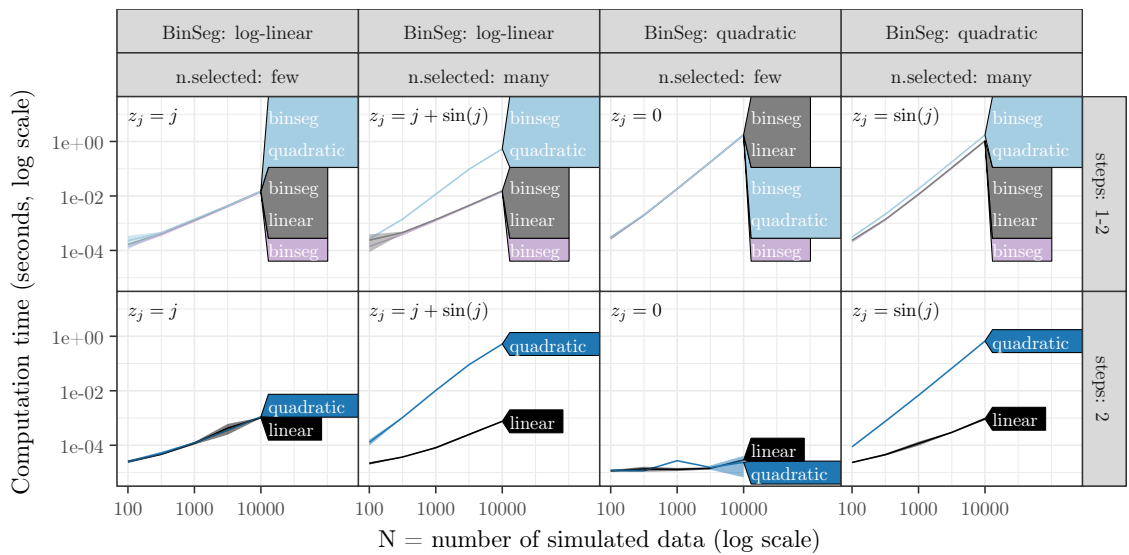}
    \caption{Binary segmentation (step 1) followed by exact model selection (step 2) was run on four synthetic data sequences $x_i$ (panels from left to right). 
    Both model selection algorithms (worst case quadratic, and proposed worst case linear time) output the exact path of selected models.
    \textbf{Bottom:} when there are few selected models (first and third columns) the quadratic algorithm achieves its best case linear time complexity; when there are many selected models (second and fourth columns) it achieves the worst case quadratic time complexity.
    \textbf{Top:} total timings over both steps show that the linear time algorithm offers substantial speedups when binary segementation achieves its best-case log-linear time complexity, and there are many selected models (second column).}
    \label{fig:binseg-quadratic-rigaill}
\end{figure*}

Our proposed algorithm takes as input a sequence of $N$ loss values, which must be computed by some other machine learning algorithm. In this section we therefore analyzed our algorithm in the context of a two-step pipeline: (i) compute the $N$ loss values, (ii) compute an exact representation of the model selection function. The overall time complexity of the two-step pipeline is determined by the slower of the two steps. If the first step is at least quadratic, then the pipeline is as well (using either linear or quadratic time model selection in the second step). However if the first step is sub-quadratic, then we expect that our linear time algorithm in the second step should result in speedups. 

\paragraph{Simulated data for which proposed linear algorithm results in speedups over previous quadratic algorithm.} For the first step we therefore use the log-linear binary segmentation algorithm, which inputs a data sequence $\mathbf z\in\mathbb R^p$, and computes an approximate solution to (\ref{eq:L_k_cpt}). The binary segmentation algorithm computes the full path of $N=p$ models with corresponding loss values $L_1,\dots, L_N$ in $O(N\log N)$ time on average \citep{binary-segmentation, Truong2018}. For each data set size $N\in\{10^2, \dots, 10^5\}$ we generate synthetic data sequences $z_j = \sin(j)+j/N$, for all $j\in\{1,\dots,N\}$. Figure~\ref{fig:empirical-time} (left) shows timings of binary segmentation alone (binseg), exact model selection algorithms alone (linear, quadratic), and two-step pipelines (binseg.linear, binseg.quadratic), on an Intel T2390 1.86GHz CPU. As expected, our proposed linear time algorithm is orders of magnitude faster than the previous quadratic time algorithm (when run alone, and also in the two-step pipeline). For example, for $N=10^5$ data, the binseg.linear pipeline takes about 3 seconds, whereas binseg.quadratic takes about 2 minutes. More generally, such timings are typical for any data for which binary segmentation runs in log-linear time, and selected models $M_N$ increases with the data set size $N$ (second column of Figure \ref{fig:binseg-quadratic-rigaill}, same as Figure~\ref{fig:empirical-time} left). However, there are other kinds of data for which our approach is no faster than the quadratic baseline (other columns of Figure \ref{fig:binseg-quadratic-rigaill}). For example, when binary segmentation runs in quadratic time, then our linear time model selection algorithm offers no speedups to the overall pipeline (third and fourth columns of Figure \ref{fig:binseg-quadratic-rigaill}). Also, since the previous (worst case quadratic) algorithm achieves its best case linear time complexity when the number of selected models $M_N$ is small/constant, then our proposed algorithm offers no speedups in this case (first columns of Figure \ref{fig:binseg-quadratic-rigaill}). Overall, we have shown that for some data sets, our linear time algorithm provides substantial speedups relative to the previous quadratic time algorithm.

\paragraph{Real data for which proposed linear algorithm is faster than grid search.} Another baseline algorithm for computing a representation of the model selection function is a naïve approximate grid search over $G$ penalties $\lambda$, which takes $O(NG)$ time. We expected this baseline to perform poorly in the context of large $N$ and large $G$, so we performed timings on a large \texttt{chipseq} data set from the UCI repository \citep{uci-ml-repo}. We first computed a regularization path of $N=287,443$ optimal changepoint models for a sequence of $p=1,656,457$ data, and then performed timings of the model selection algorithms on the resulting $N$ loss values.
We observed that our proposed linear time algorithm is always faster than approximate grid search with at least 10 grid points (Figure~\ref{fig:empirical-time}, right). 
For example the approximate grid search takes almost 2 minutes for $N=10,000$ grid points, whereas the proposed exact linear time algorithm takes only 27 milliseconds. Overall these data indicate that the proposed linear time algorithm is indeed faster than the two baselines in large data.

\section{Prediction accuracy in supervised changepoint problems}
\label{sec:accuracy}

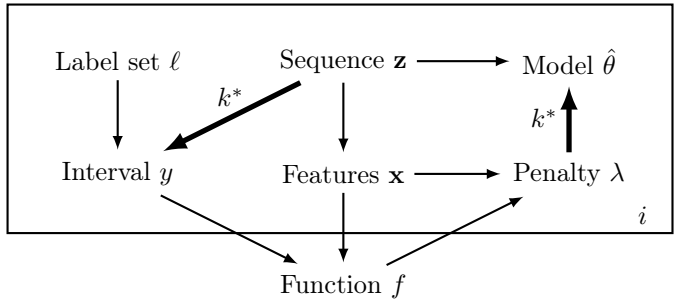
\begin{figure*}[!b]
  \centering
    \begin{tikzpicture}[->,>=latex,shorten >=1pt,auto,node distance=3cm, thick]
      \node (L) {Label set $\ell$};
      \node (y) [below of=L, node distance=1.5cm] {Interval $y$};
      \node (z) [right of=L] {Sequence $\mathbf z$};
      \node (x) [right of=y] {Features $\mathbf x$};
      \node (f) [below of=x, node distance=1.5cm] {Function $f$};
      \node (l) [right of=x] {Penalty $\lambda$};
      \node (m) [right of=z] {Model $\hat \theta$};
      \path 
      (L) edge (y)
      (z) edge [line width=2pt] node [above] {$k^*$} (y)
      (z) edge (x)
      (y) edge (f)
      (x) edge (f)
      (f) edge (l)
      (x) edge (l)
      (l) edge [line width=2pt] node {$k^*$} (m)
      (z) edge (m)
      ;
 \node[rectangle, inner sep=0mm, fit= (L) (l),label=below right: $i$, xshift=28mm] {};
  \node[rectangle, inner sep=5mm,draw=black!100, fit= (L) (l)] {};
    \end{tikzpicture}
  \caption{Computation graph for supervised changepoint detection in labeled sequences $i$. Directed edges start from inputs and end at outputs, e.g. interval $y$ can be computed using labels $\ell$ and sequence $\mathbf z$. Bold edges indicate computations which use the model selection function $k^*$.}
  \label{fig:computation-graph}
\end{figure*}

\begin{figure*}
  \centering
  \includegraphics[width=\textwidth]{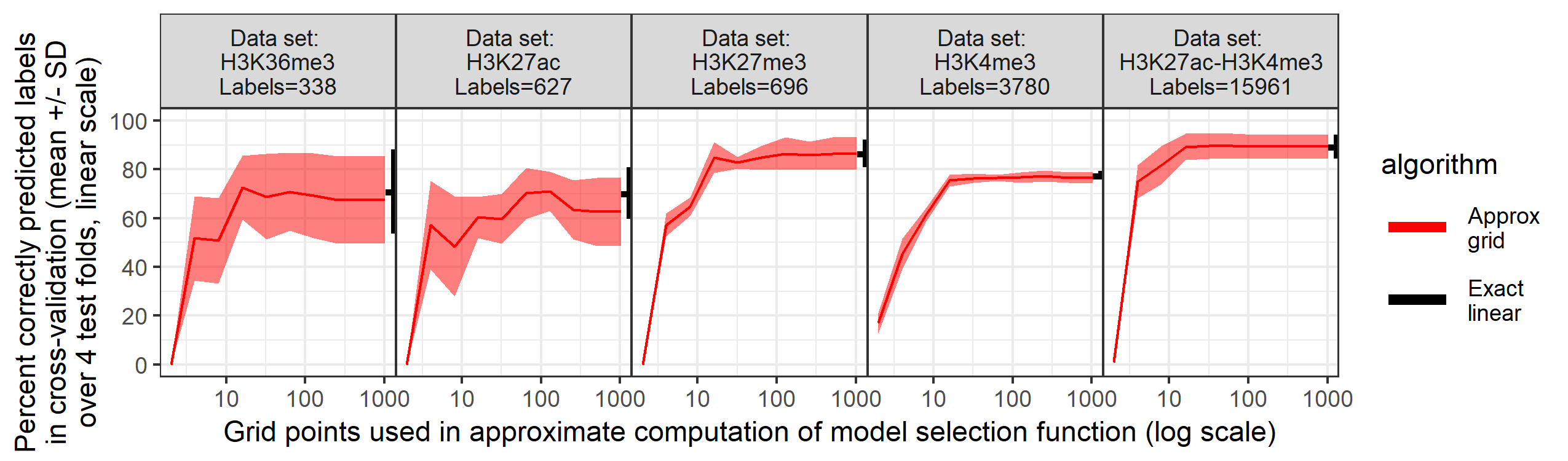}
  \vskip -0.3cm
  \caption{Prediction accuracy on held-out test data increases as a function of number of points used in approximate grid search algorithm (red line/band); it takes 10--100 grid points to achieve the maximum accuracy in each data set (panels), which is also achieved by the proposed linear time exact algorithm (black point/error bar on right).}
  \label{fig:chipseq-cv}
\end{figure*}

In this section we aim to demonstrate that the proposed exact algorithm results in more accurate predictions than a naïve approximate grid search.
To examine the accuracy of our algorithm, we consider several supervised changepoint detection problems from the UCI \texttt{chipseq} data, which contain labels that indicate presence/absence of changepoints in particular data subsets. Accurate changepoint detection in these data is important in order to characterize active/inactive regions in the human epigenome.

Here we give a brief summary of the supervised learning framework for changepoint detection; for details see \citep{HOCKING-penalties}. Each observation $i$ is represented by a numeric data vector/sequence $\mathbf z_i$ along with a corresponding label set $\ell_i$. 
We compute a feature vector $\mathbf x_i$ then learn a penalty function $f(\mathbf x_i)=\log \lambda_i$ which results in a model $\hat \theta(\lambda_i)$. The goal is to learn a function $f$ that results in minimal errors with respect to the labels $\ell_i$ in test data sequences. 
In this context there is a model selection function $k_i^*$ which is specific to each data sequence $i$, and is used in two places during the learning and prediction (bold arrows in Figure~\ref{fig:computation-graph}).
First, it is used to compute the interval/output $ y_i=(\underline y_i, \overline y_i)$ of optimal penalty values for each training data sequence $i$, such that predicting $f(\mathbf x_i)=\log\lambda_i\in (\underline y_i, \overline y_i)$ results in minimal label errors. Second, it is used to compute the predicted model $\hat \theta(\lambda_i)$ given a predicted penalty $\lambda_i$.
We learn a linear $f$ by minimizing an L1-regularized cost function \citep{HOCKING-penalties}, using outputs $y_i$ computed by either our exact algorithm or a naïve approximate grid search with a variable number of penalties $\lambda$. 

We performed 4-fold cross-validation in five different labeled data sets (panels in Figure~\ref{fig:chipseq-cv}). We observed in each data set that it takes 10--100 penalties $\lambda$ in the grid search to achieve the maximum number of correctly predicted labels, which was also achieved by the proposed exact algorithm. Overall these data provide empirical evidence that, in the context of supervised changepoint detection problems, using an exact representation of the model selection function results in more accurate predictions than using an approximate representation obtained via grid search.

\section{Discussion and conclusions}
\label{sec:discussion}

For learning problems with $\ell_0$ regularization, we proposed a new dynamic programming algorithm for computing an exact representation of the model selection function~(\ref{eq:k*}). By bounding the number of iterations, we proved theoretically that the algorithm is linear time in the worst case. In real and synthetic data we empirically validated these bounds, and showed that the proposed linear time algorithm is orders of magnitude faster than two baselines. We used cross-validation in supervised changepoint detection problems to show that the exact representation provides more accurate predictions than the grid search approximation baseline.

Our algorithm requires no special data structures and can be efficiently implemented using arrays in standard C; our free software implementation is available at \url{https://github.com/tdhock/penaltyLearning/}. For reproducibility we also provide the source code that we used to make the figures at \url{https://github.com/tdhock/changepoint-data-structure}.
For future work we would like to consider selecting models from a partial set $S\subset \{1,\dots,N\}$, and develop an efficient algorithm for updating an exact representation of the corresponding model selection function.

\bibliographystyle{abbrvnat}
\bibliography{refs}
\end{document}